\documentclass{IEEEtran}

\IEEEoverridecommandlockouts    

\usepackage{graphics} 
\usepackage{epsfig} 
\usepackage{mathptmx} 
\usepackage{times} 
\usepackage{amsmath,amsthm} 
\usepackage{amssymb} 
\usepackage{url}
\usepackage{color}
\usepackage{soul}
\usepackage{enumerate}
\usepackage{mathrsfs}
\usepackage[linesnumbered,algoruled,boxed,lined]{algorithm2e}

\usepackage{datetime}
\date{\today}

\PassOptionsToPackage{usenames,dvipsnames,svgnames}{xcolor}
\usepackage{tikz}
\usetikzlibrary{arrows,positioning,automata}

\usepackage{pgf}
\usepackage{tikz}
\usetikzlibrary{arrows,automata}
\usepackage[latin1]{inputenc}
\usepackage{dsfont} 

\newtheorem{theorem}{Theorem}
\newtheorem{corollary}{Corollary}
\newtheorem{lemma}{Lemma}
\newtheorem{definition}{Definition}

\newtheorem{remark}{Remark}
\newtheorem{proposition}{Proposition}

\newtheorem{assumption}{Assumption}

\newcommand{\Expt}{\ensuremath{\mathbb{E}}}
\newcommand{\Prob}[1]{{\ensuremath{\mathbb{P}\left[ #1 \right]} }}
\newcommand{\rp}{ \right) }
\newcommand{\lp}{ \left( }
\newcommand{\Proj}{\text{Proj}}
\newcommand{\argmin}{\text{argmin}}

\title{Differentially Private Consensus-Based \\ Distributed Optimization}
\author{Mehrdad Showkatbakhsh, Can Karakus, Suhas Diggavi
\thanks{M. Showkatbakhsh and S. Diggavi are with the UCLA Electrical and Computer Engineering Department, Los Angeles, CA
        {\tt\small \{mehrdadsh, suhasdiggavi\}@ucla.edu}.}%
\thanks{C. Karakus is with Amazon Web Services, East Palo Alto, CA
        {\tt\small cakarak@amazon.com}. The work was done while at UCLA.} 
 
}

\begin{document}
\maketitle

\begin{abstract}

Data privacy is an important concern in learning, when datasets contain sensitive information about individuals. This paper considers consensus-based distributed optimization under data privacy constraints. Consensus-based optimization consists of a set of computational nodes arranged in a graph, each having a local objective that depends on their local data, where in every step nodes take a linear combination of their neighbors' messages, as well as taking a new gradient step. Since the algorithm requires exchanging messages that depend on local data, private information gets leaked at every step. Taking $(\epsilon, \delta)$-differential privacy (DP) as our criterion, we consider the strategy where the nodes add random noise to their messages before broadcasting it, and show that the method achieves convergence with a bounded mean-squared error, while satisfying $(\epsilon, \delta)$-DP. By relaxing the more stringent $\epsilon$-DP requirement in previous work, we strengthen a known convergence result in the literature. We conclude the paper with numerical results demonstrating the effectiveness of our methods for mean estimation.

\end{abstract}

\section{Introduction}
\label{sec:introduction}

Data privacy is a central concern in statistics and machine learning when utilizing sensitive databases such as financial accounts and health-care. Thus, it is important to design machine learning algorithms which protect users' privacy while maintaining acceptable level of accuracy. In this paper, we consider a distributed framework in which $N$ nodes aim to minimize a global cost function $f(x) = \sum_{i \in [N]} f_i(x), x \in \mathcal{X}$ where $f_i$ is only available to the $i$-th node and contains sensitive information about individuals trusting this node. We study the consensus-based gradient distributed algorithm in which nodes broadcasts their local estimates and update them accordingly based on what they received from the neighbors. However, revealing the local estimate may expose privacy of individual data points and cares must be taken into account to protect the sensitive information from an adversary that oversees all messages among nodes.

Differential Privacy (DP) is a well-known notion of privacy \cite{Dwork2014} and found application in many domains (we refer readers to  \cite{Dwork2014} and \cite{Sarwate2013} ). DP assumes a strong adversary that has access to all data points except one and rigorously limits inferences of an adversary about each individual, thereby ensuring robustness of the privacy guarantee to side information. Furthermore, it does not assume any distribution on the underlying data and guarantees it gives do not depend on such assumptions. In this framework, there has been a long line of work studying differentialy private machine learning algorithms, see \cite{Sarwate2013} and references therein. Empirical Risk Minimization (ERM) plays an important role in the supervised learning setup and our work is tied to private ERM in distributed setup.

The algorithm in this work is not new and is a small modification of the (sub)-distributed gradient descent (DGD) algorithm \cite{dist_09, wotao16} which has been analyzed before in the literature of differential privacy \cite{Nitin2015}. The main novelty of our work lies in the new analysis that leads to a \emph{stronger} convergence results. Contribution of this work is as follows:

\begin{itemize}
\item We determine the variance of the noise needed to ensure DP privacy in this \textit{iterative} and \textit{distributed} setup (Theorem \ref{thm:dp}) from basic calculations, instead of using composition theorems \cite{Dwork2014}. This approach gives a tighter bound for the noise variances and let us increase the accuracy by optimizing over the variances. 
\item We derive the non-ergodic convergence behavior in this setup, thereby showing that by suitably choosing the noise variance the parameter converges to a ball around the optimal point. We further characterize the radius of the ball as a function of privacy parameters, \emph{i.e.}, $\epsilon$ and $\delta$ (Theorem \ref{thm:conv}).
\end{itemize}


\noindent\textbf{Related work.} There is long line of research devoted to differentially private ERM in the centralized set-up, in which a trusted party has access to a private and sensitive database while the adversary observes only the final end model \cite{chaudhuri2011, Smith2014}. A number of approaches exist for this set up with the convex loss function, which can be roughly classified into three categories. 
The first type of approaches is to inject properly scaled additive noise to the output of a non-DP algorithm which was first proposed by \cite{Chaudhuri_09} for this problem and later on is extended by \cite{Zhang2017} and \cite{Wu2017}. The second type of approaches is to perturb the objective function which is again introduced in \cite{Chaudhuri_09}. The third approach delves into the first order optimization algorithms and perturbs gradients at each step to maintain the DP, \cite{Smith2014} was one of the earliest work in this domain.

In our work, there does not exist a central trusted entity and data is distributed among $N$ nodes that motivates the use of the distributed optimization algorithms. Differentially private distributed optimization has been explored before in \cite{Nitin2015}, where authors considered the similar problem under $\epsilon$-DP constraint. It is well-known that $\epsilon$-DP is too stringent condition and often rises to non-acceptable accuracy. The convergence bound of \cite{Nitin2015} does not diminish as the privacy requirement weakens, \emph{i.e.}, the convergence is not exact even without any privacy requirement. We emphasize that the main novelty of our work lies in the analysis of the algorithms and the resulting theoretical guarantees.

\noindent \textbf{Paper organization}. In section \ref{sec:background}, we give a brief overview of differential privacy followed by introducing the problem. Section \ref{sec:mainresults} introduces the main results in which we establish the condition under which the distributed algorithm is differentially private and the convergence results. In section \ref{sec:analysis}, we give the proof outline of Theorems \ref{thm:dp} and \ref{thm:conv}. We demonstrate our numerical experiments in Section \ref{sec:experiments}. Section \ref{sec:conclusion} concludes the paper.

\section{Background and Problem Formulation}
\label{sec:background}

In this section, we review the notion of differential privacy and  the Gaussian mechanism which are building block of our algorithm. In the second part, we give the precise problem formulation along with the overview of the algorithm which we consider in this work.
\subsection{Differential Privacy}
Let $D := \{ d_1, \cdots, d_N \}$ be a database containing $N$ points in the universe $\mathbb{D}$. Two databases $D$ and $D'$ are called neighbors when they differ in \emph{at most} one data point, we use the notation $D' \sim D$ to denote this relation.

A randomized mechanism $\mathcal{M}$ is differentially private if evaluated on $D$ and $D'$ produces outputs that have similar statistical distributions. Formally speaking,

\begin{definition}[$(\epsilon, \delta)$-Differential Privacy]
A randomized mechanism $\mathcal{M}$ is $(\epsilon, \delta)$-differentially private, if for any $S \subseteq \mbox{Range}(\mathcal{M})$,
\begin{align}
\Prob{ \mathcal{M}(D) \in S } \leq e^{\epsilon} \Prob{ \mathcal{M}(D') \in S } + \delta. \label{eq:def:dp}
\end{align}
\end{definition}

An equivalent characterization of $(\epsilon, \delta)$-DP can be stated based on the tail bound on the \emph{privacy loss random variable} that is the log ratio of the probability density functions\footnote{In this work, we assume the induced measures are absolutely cts wrt to the Lebesgue so pdf always exists.} of $\mathcal{M}(D)$ and $\mathcal{M}(D')$.

\begin{proposition}[See Lemma 3.17 in \cite{Dwork2014}]
A randomized mechanism $M$ is $(\epsilon, \delta)$-differentially private if the log-likelihood ratio when evaluating on two neighboring databases remains bounded with probability at least $1-\delta$, i.e.,
\begin{align}
\Prob{ | \log \frac{ \mbox{pdf}_{D}(o) }{ \mbox{pdf}_{D'}(o) } | \leq \epsilon  } \geq 1-\delta, \label{eq:prop:dp:pdf}
\end{align} where $\mbox{pdf}_{D}$ ($\mbox{pdf}_{D'}$) is the pdf of $\mathcal{M}(D)$ ($\mathcal{M}(D')$) and $o$ is drawn according to $\mbox{pdf}_{D}$. \label{prop:dp:equiv}
\end{proposition}

A common design paradigm to approximate a deterministic function $q : \mathbb{D}^{|D|} \to R^p$ with a differentailly private mechanism is by adding a properly scaled Gaussian noise to the output of $q$. The scale of the noise depends on how far that query maps two neighboring databases which is formalized through the notion of sensitivity.

\begin{definition}[$L_2$ sensitivity]
Let $q$ be a deterministic function that maps a database to a vector in $\mathbb{R}^p$. The $L_2$ sensitivity of $q$ is defined as
\begin{align}
\Delta_2(q) &= \max_{D' \sim D} \| q(D) - q(D')  \|_2.
\end{align} We can define the $L_1$ sensitivity similarly.
\end{definition}

Throughout this work, we focus on Gaussian Mechanism to ensure differential privacy.

\begin{definition}[Gaussian mechanism] Given any (deterministic) function $q : \mathbb{D}^{|D|} \to R^p$, the Gaussian Mechanism is defined as:
\begin{align}
\mathcal{M}(D, q, \sigma) = q(D) + n,
\end{align}
where $n$ is a Gaussian random variable with a zero mean and the variance of $\sigma$.
\end{definition}
It is well known \cite{Dwork2014} that for the proper value of the noise variance, Gaussian Mechanism preserves $(\epsilon, \delta)$-DP.
\begin{proposition}[See for example Theorem 3.22 in \cite{Dwork2014}]
For a deterministic function $q : \mathbb{D}^{|D|} \to R^p$, Gaussian mechanism $\mathcal{M}(D, q, \sigma)$ preserves $(\epsilon, \delta)$-DP for $\epsilon < 1$ if $\sigma^2 \geq 2 {\log( {1.25}/{\delta} ) \Delta_2 (q)}/{\epsilon^2}$, where $\Delta_2(q)$ is the $L_2$-sensitivity of the function $q$.
\end{proposition}


\subsection{Problem Formulation}

We consider a distributed optimization set-up where $N$ nodes aim to collaboratively minimize an additive cost function $f(x) = \sum_{i \in [N]} f_i(x), x \in \mathcal{X}$ where $\mathcal{X}$ is the domain of the problem and $[N] \stackrel{\triangle}= \{1, \hdots, N\}$. In this problem, nodes want to minimize $f(x)$ while keeping each data point private. We adopt the $(\epsilon, \delta)$-differential privacy (DP) as a measure of the privacy.

\begin{assumption}[Domain]
The domain of the optimization $\mathcal{X} \subseteq \mathbb{R}^p$ is a closed compact and convex set and  $x^\ast \in \mathcal{X}$ where $x^\ast \in \argmin_{x \in \mathbb{R}^p} f(x)$. \label{asump:domain}
\end{assumption}
In this setup, $N$ nodes communicate over a connected and undirected graph $\mathcal{G} := (\mathcal{V}, \mathcal{E})$ where $\mathcal{V} := \{1, \cdots, N\}$ is the set of vertices and $\mathcal{E} \subseteq \mathcal{V} \times \mathcal{V}$ denotes the set of edges. Nodes can only communicate with their neighbors, which we denote neighbors of node $i$ with $\mathcal{N}_i$ for $i \in [N]$. We assume that each node has access to one of the summands of the global objective function, $f_i(x)$. Throughout this paper, the following assumption holds for the local objective functions\footnote{Assumption \ref{asump:smooth} and \ref{asump:strong_conv} inherently state that the smoothness and strong convexity parameters should be the same, however, without loss of generality we can take the maximum $L$ and minimum $\mu$ across all nodes.}.

\begin{assumption}[Bounded Gradient]
$f_{i}(x)$ for $i \in [N]$ are $G$-Lipschitz for $x \in \mathcal{X}$, \label{asump:bddgrad}
\begin{align}
    \| f_i(x) - f_i(y) \|_2 \leq G \|x - y\|_2,\mbox{ } \forall x,y \in \mathcal{X}, \mbox{ } i \in [N].
\end{align}
\end{assumption}

\begin{assumption}[Smoothness]
$f_i(x)$ for $i \in [N]$ are $L$-smooth over an open set containing $\mathcal{X}$, \label{asump:smooth}
\begin{align}
    \| \nabla f_i(x) -\nabla f_i(y) \|_2 \leq L \|x - y\|_2,\mbox{ } \forall x,y \in \mathcal{X}, \mbox{ } i \in [N].
\end{align}
\end{assumption}

In order to have convergence of the parameter itself, it is well-known that the strong convexity of the cost functions are needed.

\begin{assumption}[Strong convexity]
$f_i(x)$ for $i \in [N]$ are $\mu$-strongly convex over an open set containing $\mathcal{X}$, \label{asump:strong_conv}
\begin{align}
    \langle f_i(x) - f_i(y), x - y \rangle \geq \mu \|x - y\|_2^2,\mbox{ } \forall x,y \in \mathcal{X}, \mbox{ } i \in [N].
\end{align}
\end{assumption}

In the context of empirical risk minimization local cost functions are the empirical risk associated with data points stored in nodes, \emph{i.e.},
\begin{align*}
    f_i(x) := \sum_{d \in D_i} l(x; d),
\end{align*} where $D_i \subseteq \mathbb{D}$ is the set of sensitive data points stored in the $i$-th node and $l(x , d)$ is the loss function. 

\noindent \textbf{Adversary model}. In this problem, the adversary can overhear all the messages between nodes without any computational assumption. Nodes aim to preserve DP with respect to the sensitive data points $\cup_{i \in [N] D_i}$.

\subsection{Overview of the algorithm}

We study the consensus-based distributed optimization algorithm where nodes update their local estimate by combining the information received from the neighbors. As opposed to the standard Distributed Gradient Descent (DGD), our proposed algorithm consists of two phases.

\noindent \textbf{Stage I.} In the first stage of the algorithm, similarly to the distributed gradient descent, nodes iteratively perform the consensus step followed by a local Gradient Descent (GD) step. Let us denote the local estimate of the $i$-th node with $x_i(t)$.
\begin{align}
x_i(t) &= \Proj_{\mathcal{X}} \lp z_i(t) - \eta_t \nabla f_i (z_i(t)) \rp, \hspace{.015\textwidth} t \leq T  \label{eq:def:GD}
\end{align} where 
\begin{align}
z_i(t) &= \Proj_{\mathcal{X}} \lp \sum_{j \in \mathcal{N}_i} w_{i j} y_j(t)\rp, \label{eq:def:consensus}
\end{align} here $y_j(t)$ is the message sent by node $j \in \mathcal{N}_i$ to its neighbors, and 
\begin{align}
\Proj_{\mathcal{X}}(x) \stackrel{\triangle}= \argmin_{y \in \mathcal{X}} \|x - y \|_2 \label{eq:def:proj}    
\end{align}
is the Euclidean projection onto the set $\mathcal{X}$. The update \eqref{eq:def:GD} shows that node $i$ updates it's local estimate by taking a proper average of the messages sent by its neighbors and descending it through $\nabla f_i (z_i(t))$, $w_{ij}$ is the weight associated to neighbors of node $i$ and it's zero for $j \notin \mathcal{N}_i$, and $T \in \mathbb{N}$ is the number of steps in Stage II.

\begin{assumption}[Doubly Stochastic Weight Matrix]
The weight matrix, $W := [w_{i j}]$ is a doubly stochastic matrix with non-negative entries. We denote the second largest eigenvalue of $W$ in absolute value with $\beta := \max \{ |\lambda_2(W)|, |\lambda_N(W)| \}$, where $1 = \lambda_1(W) \geq \lambda_2(W) \hdots \geq \lambda_N(W) > -1$ are eigenvalues of $W$ sorted in a descending order. \label{asump:weightmatrix}
\end{assumption}
It is well-known that using a fixed step size the classical DGD converges to a neighborhood of the optimal point with the size proportional to the step-size \cite{dist_09, wotao16}. Convergence to the exact point can be derived by using a diminishing step size (see \cite{wotao16} and references therein). Therefore throughout this work, step-size $\eta_t$ is chosen $\Theta(\frac{1}{t})$.

As opposed to the classical DGD, nodes do not send their local estimate directly since each step of GD may reveal information about the underlying sensitive data points. Nodes perturb their local estimate by a Gaussian mechanism to control the privacy leakage. In particular, nodes broadcast
\begin{align}
y_i(t+1) = x_i(t) + n_i(t), \quad i \in \{1, \cdots, N\}, t \leq T \label{eq:def:noise},
\end{align} where $n_i(t)$ is a zero mean additive Gaussian noise with the variance of $M^2_t$, and $x_i(t)$ is the local estimate of node $i$ at $t$. 

Each step of GD exposes the sensitive data points to the adversary. Hence to ensure the same level of privacy additional noise needed as the number of GD steps increases. Noise added to each stage of GD avoids the local estimates to converge to a common value therefore in our proposed method nodes apply GD only for $T$ steps, and afterwards nodes switch to a purely consensus mode to agree on a common value.

\noindent \textbf{Stage II.} After $T$ steps of Gradient descent, nodes iterates only though the consensus steps. Note that $\nabla f_i$ in \eqref{eq:def:GD} is the only source in which the privacy of sensitive data points may leak. Therefore in the second stage of the algorithm, nodes broadcast their local estimate precisely and update their local estimate according to:
\begin{align}
y_i(t) &= x_i(t-1), \quad \quad  i \in [N], \quad  t > T   \nonumber \\ 
x_i(t) &=  \sum_{j \in [N]} w_{i j} y_j(t), \label{eq:def:stage2:consensus}
\end{align}.

Note that projection operator is not needed in \eqref{eq:def:stage2:consensus} due to the convexity of the optimization domain (Assumption \ref{asump:domain}). Nodes update their local estimates until a stopping criterion is met. One common stopping criterion is the relative change in the value of each node.

\begin{algorithm}[htbp]
\SetAlgoLined

 Set $y_i(1) = 0$\;
 \For{$t = 1, \cdots, T$}{
 Update $y_i(t)$ according to \eqref{eq:def:noise} and broadcast $y_i(t)$\;
 Receive $y_j(t)$ from $j \in \mathcal{N}_j$ \;
 Update $x_i(t)$ according to \eqref{eq:def:GD} \;
}

\For{$t = T+1, \cdots$}{
 Broadcast $y_i(t) := x_i(t-1)$ \;
 Receive $y_j(t)$ from $j \in \mathcal{N}_j$ \;
 Update $x_i(t)$ according to \eqref{eq:def:stage2:consensus} \;
}

\caption{Steps at Node $i$ for the proposed private distributed algorithm.}

\label{alg:PDGD}
\end{algorithm}

\noindent \textbf{Notation.} We represent set of natural and real numbers with $\mathbb{N}$ and $\mathbb{R}$ respectively. Throughout this work, we reserve the lowercase bold letters for the aggregated parameters of nodes at a given time $t$, \emph{i.e.}, we use the notation $\mathbf{x}(t) = [x_1(t); \cdots; x_N(t)]$, $\mathbf{y}(t) = [y_1(t); \cdots; y_N(t)]$, $\mathbf{z}(t) = [z_1(t); \cdots; z_N(t)]$ and $\mathbf{n}(t) = [n_1(t); \cdots; n_N(t)]$. The identity matrix is denoted with $I_{p} \in \mathbb{R}^{p}$ and we use $\mathbf{1_N}$ to represent a vector of length $N$ of ones. We use $\| a \|$ to denote $l_2$-norm of a vector $a$ while $\|A\|$ represents the operator norm of a matrix $A$, $\|A\| \stackrel{\triangle}= \sup_{\|x\|=1} \|Ax\|$, and $\otimes$ denotes the Kronecker product. In this paper we reserve the notation $[N]$ to represent $\{1, \cdots, N\}$ for any $N \in \mathbb{N}$.

\section{Main results}
\label{sec:mainresults}
In this section, we give the main results of this work. First we derive conditions on the noise variances under which the distributed problem is $(\epsilon, \delta)$-differentially private against an adversary that oversees all the communications among nodes. We emphasize that one of the main contribution of this work is to relate the noise variances at different steps to the privacy parameters ($\epsilon$ and $\delta$) directly rather than using basic or advanced composition theorems. Afterward we state the convergence result.

The variance of the noise to ensure $(\epsilon, \delta)$-DP depends on the $L_2$ sensitivity of the algorithm. In the context of distributed optimization, we need to make sure DP is satisfied despite multiple rounds of communications in this iterative scheme.

\begin{definition}[Conditional $L_2$-sensitivity]
Conditional $L_2$-sensitivity at round $t$, $\Delta(t)$ defined to be the maximum $L_2$-norm difference of ${x}_{i}(t)$ evaluated on two neighboring databases $D$ and $D'$ while having the same set of messages $\{ \mathbf{y}(k) \}_{k = 1}^{t}$ up until round $t$, \emph{i.e.},
\begin{align}
\Delta(t) := \sup_{D \sim D'} \sup_{i \in [N]} \sup_{\{ \mathbf{y}(t) \}_{t = 1}^{t}} \| {x}_i^D(t) - {x}_i^{D'}(t) \|_2,  \label{eq:def:conditionalsens}  
\end{align}
where ${x}_i^{D'}(t)$ corresponds to the local parameter of node $i$ of the neighboring instance of the problem. \label{def:conditionalsens}
\end{definition}

\begin{proposition}
The conditional $L_2$-sensitivity of Algorithm \ref{alg:PDGD} at round $t \leq T$ is bounded by $2 \eta_t G$, provided that Assumption \ref{asump:bddgrad} holds. \label{prop:conditionalsens}
\end{proposition}
\begin{proof}
See Appendix.
\end{proof}
Now we have the machinery to state the main theorem of this section.
\begin{theorem}
The distributed algorithm is $(\epsilon, \delta)$-DP if nodes perturb their local estimates by adding independent Gaussian noise \eqref{eq:def:noise} and the following holds,
\begin{align}
\sum_{t = 1}^{T} \frac{\Delta^2(t)}{M_t^2} &\leq  \frac{\epsilon^2}{\epsilon + 2 \log \frac{2}{\delta}}, \label{eq:thm:dp}
\end{align} where $M_t$ is the scale of noise added in round $t \leq T$. \label{thm:dp}
\end{theorem}

\begin{corollary}
If Assumption \ref{asump:bddgrad} holds and
\begin{align}
\sum_{t = 1}^{T} \frac{\eta_t^2}{M_t^2} \leq \frac{\epsilon^2}{4 G^2 \lp \epsilon + 2 \log \frac{2}{\delta} \rp} \stackrel{\triangle}= \kappa(\epsilon, \delta). \label{eq:privacy}
\end{align}
then the distributed algorithm is $(\epsilon, \delta)$-DP.
\end{corollary}

\begin{remark}
A common practice to ensure differential privacy in an iterative mechanism is to make each step differentially private (with a stronger privacy guarantees) and combine the privacy leakage using the basic or advanced composition theorems \cite{Dwork2014}. Composition theorems do not take into account the specific noise distribution and they often give loose results for a given distribution while Theorem \ref{thm:dp} takes the noise distribution into account thereby giving a tighter result, \emph{i.e.}, smaller noise variances and therefore ensure a better utility.
\end{remark}

\begin{remark}
In the literature of DP for iterative processes, it is common that in order to ensure $(\epsilon, \delta)$-DP we make each step $(\epsilon', \delta')$-DP, where $\epsilon'$ and $\delta'$ are computed using a composition theorem. It implicitly assumes that we need to have the same privacy requirement at each step, which is not necessary needed. Theorem \ref{thm:dp} connects the noise variances across time to the privacy parameters $\epsilon$ and $\delta$ directly and allows for a meaningful assignment of privacy budget to different steps.
\end{remark}

\begin{remark}
In the regime where $\epsilon \ll 1$ and $\delta \ll \frac{1}{N}$, the bound in Theorem \ref{thm:dp} can be written as,
\begin{align}
    \sum_{t = 1}^{T} \frac{\Delta^2(t)}{M_t^2} \leq \frac{\epsilon^2}{2 \log \frac{2}{\delta}}.
\end{align}
\end{remark}

Theorem \ref{thm:dp} extends Lemma 2 the result of \cite{Nitin2015} to $(\epsilon, \delta)$-DP. It gives us the condition under which the distributed algorithm is $(\epsilon, \delta)$-DP. Working with $(\epsilon, \delta)$-DP as opposed to $\epsilon$-DP in \cite{Nitin2015}, enables us to derive a convergence result with a diminishing regret bound when the privacy requirement weakens.

In the rest of this section, we state the convergence result. Let us denote the average of the local estimates with $\bar{x}(t) \stackrel{\triangle}= \frac{1}{N} \sum_{i \in [N]} x_i(t)$. Theorem \ref{thm:conv} summarizes the convergence result for the mean parameter $\bar{x}(t)$.

\begin{theorem}
Under Assumptions \ref{asump:bddgrad}, \ref{asump:smooth} and \ref{asump:strong_conv} with the step size $\eta_t = \frac{\mu + L}{2 \mu L} \frac{1}{t}$ and noise scales $M_t^2 = \frac{2}{\kappa(\epsilon, \delta)} \lp \frac{\mu + L}{2\mu L} \rp^2 \frac{\sqrt{T}}{t \sqrt{t}}$, the distributed algorithm \ref{alg:PDGD} is differentially private and the following bound holds on $ \Expt[\| \bar{x}(T) - x^\ast \|_2^2] $:
\begin{align}
\Expt[\| \bar{x}(T) - x^\ast \|_2^2] \leq C_{T} \frac{1}{T} 
+ C_{\log T} \frac{\log T}{T} 
+ C_{\sqrt[4]{T}}  \frac{1}{\sqrt[4]{T}} 
+ C_{(\epsilon, \delta)} \nonumber
\end{align} where $x^\ast$ minimizes, $C_T$, $C_{\log T}$, $C_{\sqrt[4]{T}}$ and $C_{(\epsilon, \delta)}$ are constants:
\begin{align*}
C_{T} &= \frac{S(0)}{N} \\ 
C_{\log T} &= G^2 \lp 1 + \frac{1}{1-\beta} \rp \lp \frac{\mu + L}{2\mu L} \rp^2 \\
C_{\sqrt[4]{T}} &= \frac{2 \sqrt{2p} G }{\sqrt{\kappa(\epsilon,\delta)}} \lp 4 + \frac{3}{1-\beta} \rp \lp \frac{\mu + L}{2\mu L} \rp^2 \\
C_{(\epsilon, \delta)} &= \frac{2p}{\kappa(\epsilon, \delta)} \lp \frac{\mu + L}{2\mu L} \rp^2 \\
\kappa(\epsilon, \delta) &= \frac{\epsilon^2}{4 G^2 \lp \epsilon +  2 \log \frac{2}{\delta} \rp} 
\end{align*} \label{thm:conv}
\end{theorem}

Theorem \ref{thm:conv} states that the average parameter converges to a neighborhood of the optimal point with the rate of ${O}(\frac{1}{\sqrt[4]{T}})$. The neighborhood scales with the privacy parameters $(\epsilon, \delta)$ which is ${O}(\frac{\log \frac{1}{\delta}}{\epsilon^2})$. Recall the second stage of the distributed algorithm only consists of the consensus steps in which local parameters converge to a common value in a linear rate.

\begin{corollary}
Under Assumptions \ref{asump:bddgrad}, \ref{asump:smooth} and \ref{asump:strong_conv} with the step size $\eta_t = \frac{\mu + L}{2 \mu L} \frac{1}{t}$ and noise scales $M_t^2 = \frac{2}{\kappa(\epsilon, \delta)} \lp \frac{\mu + L}{2\mu L} \rp^2 \frac{\sqrt{T}}{t \sqrt{t}}$, the distributed algorithm \ref{alg:PDGD} is deferentially private and the following bound holds on the local parameters in the second stage of the algorithm, $t > T$:
\begin{align}
\Expt[\| &x_i(t) - x^\ast \|_2^2] \leq  \\ 
& 2 C_{{exp}} \beta^{2t-2T} + 2 C_{T} \frac{1}{T} 
+ 2 C_{\log T} \frac{\log T}{T} 
+ 2 C_{\sqrt{T}}  \frac{1}{\sqrt{T}} 
+ 2 C_{(\epsilon, \delta)} \nonumber
\end{align} where $C_T$, $C_{\log T}$, $C_{\sqrt[4]{T}}$ and $C_{(\epsilon, \delta)}$ defined in Theorem \ref{thm:conv} and $C_{exp} = 2 \| \mathbf{x}(T) \|^2 $.
\end{corollary}
\begin{proof}

We observe that in Stage II, the mean parameter $\bar{x}(t)$ remains constant since
\begin{align}
    \bar{x}(t+1) &\stackrel{(a)}= \frac{1}{N} \lp \mathbf{1}_N^T \otimes I_P \rp \mathbf{x}(t+1) \nonumber \\
    &\stackrel{(b)}= \frac{1}{N} \lp \mathbf{1}_N^T \otimes I_P \rp \lp  W \otimes I_p \rp  \mathbf{x}(t) \nonumber \\
    &\stackrel{(c)}= \frac{1}{N} \lp \mathbf{1}_N^T \otimes I_P \rp \mathbf{x}(t) = \bar{x}(t), \label{eq:cor:conv:mean1}
\end{align} where we rewrote the mean parameter using the Kronecker product in $(a)$, $(b)$ follows directly from \eqref{eq:def:stage2:consensus} and we used double stochasity of $W$ in $(c)$. Now we are ready to conclude the result, 
\begin{align}
    \Expt[\| x_i(t) - x^\ast \|_2^2] &\stackrel{(a)}= \Expt[\| x_i(t) - \bar{x}(t) + \bar{x}(T) - x^\ast \|_2^2] \label{eq:thm:conv:corr} \\
    &\stackrel{(b)}\leq 2\Expt[\| x_i(t) - \bar{x}(t) \|_2^2 ] + 2\Expt[ \| \bar{x}(T) - x^\ast \|_2^2 ], \nonumber
\end{align} where $(a)$ follows from \eqref{eq:cor:conv:mean1}, and we used the inequality $\|a + b \|_2^2 \leq 2 \| a \|_2^2 + 2 \| b \|_2^2$ in $(b)$. Using Theorem \ref{thm:conv} and Lemma \ref{lemma:dist_mean} it is straightforward to conclude the result.
\end{proof}

\section{Privacy and Convergence Analysis}
\label{sec:analysis}
In this section we outline the proofs for Theorems \ref{thm:dp} and \ref{thm:conv} followed by explanation and intuition.

\subsection{Proof of Theorem \ref{thm:dp}}
In the context of differential privacy, the corresponding mechanism for the distributed algorithm maps $D := \cup_{i \in [N]} D_i$ to a sequence of messages $\{\mathbf{y}(t) \}_{t = 1}$. In order to satisfy $(\epsilon, \delta)$-DP, the output of the mechanism should satisfy the condition \eqref{eq:prop:dp:pdf} in Proposition \ref{prop:dp:equiv}. We proceed by writing the privacy loss random and bounding it using the concentration inequalities. The complete proof is included in Appendix.

\subsection{Proof of Theorem \ref{thm:conv}}

It is straightforward to verify this choice of $\eta_t$ and $M_t$ satisfy \eqref{eq:prop:dp:pdf} and therefore the distributed algorithm is deferentially private. The proof of convergence consists of two parts. First we show that the local parameters are bounded away from mean in expectation, which is depicted in Lemma \ref{lemma:dist_mean}. The proof proceeds by bounding deviation of the mean parameter to the optimal point. Putting these together, the result follows.

\begin{lemma}
Under Assumption \ref{asump:bddgrad}, at round $t$, the following bound holds on the distance of local parameters to the mean for $t < T$,
\begin{align}
 \| \mathbf{z}(t) - \mathbf{1}_{N}  \otimes \bar{z}(t) \| \leq \| \mathbf{n}(t) \| &+ 2\sum_{s=1}^{t-1}  \beta^{t-s} \| \mathbf{n}(s) \|  \\
 &+ \sqrt{N} G \sum_{s=1}^{t-1} \eta_s  \beta^{t-s}, \nonumber
\end{align}
where $\bar{z}(t) \stackrel{\triangle}= \frac{1}{N} \sum_{i \in [N]} z_i(t)$. And the following holds for $t \geq T$,
\begin{align}
  \| &\mathbf{x}(t) -  \mathbf{1}_{N} \otimes \bar{x}(t) \| \leq \beta^{t-T} \| \mathbf{x}(T) \|.
\end{align}
where $\otimes$ denotes the Kronecker product. \label{lemma:dist_mean}
\end{lemma}
\begin{proof}
See Appendix.
\end{proof}

Let us define $S(t) \stackrel{\triangle}= \sum_{i \in [N]} \Expt[ \| x_i(t) - x^\ast \|^2 ]$ where $x^\ast$ is the global minimum of $f(x)$. Observe that,
$$\Expt[\| \bar{x}(t) - x^\ast \|^2] \leq \frac{1}{N} S(t), $$ where we used $\lp \sum_{i \in [N]} \| a_i \| \rp^2 \leq N \sum_{i \in [N]} \| a_i \|^2 $. Therefore we proceed by bounding $\Expt[ \| x_i(t) - x^\ast \|^2 ]$ for $i \in [N]$ in order to bound $\Expt[\| \bar{x}(t) - x^\ast \|^2]$.

Using the standard techniques, we bound terms $\|x_{i}(t) - x^{\ast} \|^2$ one by one (for the clarity of the presentation, the time index dropped wherever it is clear from the context). It is well known that the projection operator is \emph{non-expansive}, \emph{i.e.}, $\| \Proj_{\mathcal{X}}(x) - \Proj_{\mathcal{X}}(y) \| \leq \|x - y\|$ for $x, y \in \mathbb{R}^p$. Putting this together with Assumption \ref{asump:domain} ($x^\ast \in \mathcal{X}$) we have,
\begin{align}
\|x_{i}(t) - x^{\ast} \|^2 
&= \| \Proj_{\mathcal{X}} \lp z_{i} - \eta_t \nabla f_i(z_i) \rp - \Proj_{\mathcal{X}} \lp x^{\ast} \rp  \|^2 \nonumber \\
&\leq \|  z_{i}(t) - \eta_t \nabla f_i(z_i) - x^{\ast} \|^2  \label{eq:thm:conv:ind:1}
\end{align} In order to bound RHS of \eqref{eq:thm:conv:ind:1} we first present a lemma.
\begin{lemma}[see for example Theorem 2.1.12 in \cite{nesterov2007}]
Suppose that $f$ is $L$-smooth and $\mu$-strongly over an open set containing $\mathcal{X}$, then we have,
\begin{align}
    \langle x - y, \nabla f(x) - \nabla f(y) \rangle \geq c_1 \| x - y \|^2 + c_2 \| \nabla f(x) - \nabla f(y) \|^2, \label{eq:lemma:cocoercivity}
\end{align} where $c_1 = \frac{\mu L}{\mu + L}$ and $c_2 =  \frac{1}{\mu + L}$. \label{lemma:cocoercivity}
\end{lemma} 

We proceed by expanding \eqref{eq:thm:conv:ind:1} and by adding and subtracting $\nabla f_i(x^\ast)$, 
\begin{align}
\|x_{i}&(t) - x^{\ast} \|^2  \nonumber \\
 &\leq \| z_i - x^{\ast} \|^2 - 2 \eta_t \langle z_i - x^{\ast}, \nabla f_i(z_i) - \nabla f_i(x^\ast) + \nabla f_i(x^\ast) \rangle \nonumber \\ & \hspace{0.32\textwidth} + \eta_t^2 \| \nabla f_i(z_i) \|^2 \nonumber \\
 &\stackrel{(a)}\leq (1-  \frac{2 \mu L }{ \mu+L} \eta_t) \| z_i - x^{\ast} \|^2 - \frac{2 \eta_t}{\mu+L} \| \nabla f_i(z_i) - \nabla f_i(x^\ast) \|_2^2 \nonumber \\ 
 & \hspace{0.14\textwidth} + 2\eta_t \langle \nabla f_i(x^\ast), x^{\ast} - z_i \rangle + \eta_t^2 \| \nabla f_i(z_i) \|^2 \nonumber \\
 &\stackrel{(b)}\leq (1-  \frac{2 \mu L}{ \mu+L} \eta_t) \| z_i - x^{\ast} \|^2 + \eta_t^2 G^2 + 2 \eta_t \langle \nabla f_i(x^\ast), x^{\ast} - z_i \rangle \nonumber  \\
 &\stackrel{(c)}\leq (1-  \frac{2 \mu L}{ \mu+L} \eta_t) \| z_i - x^{\ast} \|^2 + \eta_t^2 G^2 + 2 \eta_t \langle \nabla f_i(x^\ast), x^{\ast} - \bar{z} \rangle \nonumber \\
 & \hspace{.23\textwidth}-  2 \eta_t \langle \nabla f_i(x^\ast), z_i - \bar{z} \rangle					\label{eq:thm:conv:ind:2}
\end{align} where $(a)$ follows from Lemma \ref{lemma:cocoercivity}. We used Assumption \ref{asump:bddgrad} for $(b)$, and $(c)$ comes from adding and subtracting $\bar{z}$. The following lemma is useful in order to connect \eqref{eq:thm:conv:ind:2} to $S(t)$.

\begin{lemma}
For any (fixed) $x \in \mathcal{X}$ the following holds,
\begin{align}
 \sum_{i \in [N] } \Expt[ \| z_i(t) - x \|^2 ]  \leq \sum_{i \in [N] } \Expt[ \| x_i(t-1) - x \|^2 ] + d N M_{t-1}^2.
\end{align} \label{lemma:sum:noise}
\end{lemma}
\begin{proof}
See Appendix.
\end{proof}

By summing up both sides of \eqref{eq:thm:conv:ind:2} across all nodes and Lemma \ref{lemma:sum:noise} we have:
\begin{align}
S(t) \leq ( 1 - \frac{2 \mu L }{\mu+L} \eta_t ) S(t-1) + d N ( 1 - \frac{2 \mu L }{\mu+L} \eta_t ) M^2_{t-1}& \label{eq:thm:conv:mean} \\ 
 \hspace{0.06\textwidth} + N \eta_t^2 G^2 + 2 \eta_t \sum_{i \in [N]} \Expt[ \nabla \langle \nabla f_i(x^\ast),\bar{z} - z_i \rangle ]&, \nonumber
\end{align} where we used the fact that $x^\ast$ is the global minimum and therefore $\sum_{i \in [N]} \nabla f_i(x^\ast) = 0$. It remains to bound the last term in RHS of \eqref{eq:thm:conv:mean}. Note that,
\begin{align}
\sum_{i \in [N]} \langle \nabla f_i(x^\ast),\bar{z} - z_i \rangle &\stackrel{(a)}\leq G \sum_{i \in [N]} \| \bar{z}(t) - z_i(t) \| \nonumber \\
&\stackrel{(b)}\leq G \sqrt{N} \| \mathbf{z}(t) - \mathbf{1}_N \otimes \bar{z}(t) \|, \label{eq:thm:conv:ind:3}
\end{align} where $(a)$ follows from Cauchy-Schwartz inequality along with Assumption \ref{asump:bddgrad} and we used $\lp \sum_{i \in [N]} \| a_i \| \rp^2 \leq N \sum_{i \in [N]} \| a_i \|^2 $ in $(b)$. By applying Lemma \ref{lemma:dist_mean} and taking expectation from both sides of \eqref{eq:thm:conv:ind:3} we have,
\begin{align}
    \sum_{i \in [N]} \Expt[ \nabla \langle \nabla & f_i(x^\ast),\bar{z} - z_i \rangle ] \label{eq:thm:conv:ind4} \\
    &\leq \sqrt{p N} M_t + 2 \sqrt{p} \sqrt{N} \sum_{s=1}^{t-1} M_s \beta^{t-s} + G \sqrt{N} \sum_{s=1}^{t-1} \eta_s. \nonumber
\end{align} Together with \eqref{eq:thm:conv:mean} we have the following recursion for $S(t)$,
\begin{align}
S(t) \leq ( &1 - \frac{2 \mu L }{\mu+L} \eta_t ) S(t-1) + p N ( 1 - \frac{2 \mu L }{\mu+L} \eta_t ) M^2_{t-1} \label{eq:thm:conv:mean2} \\
&+ N \eta_t^2 G^2 \nonumber \\
&+ G \sqrt{p} N M_t + 2 G \sqrt{p} N \sum_{s=1}^{t-1} M_s \beta^{t-s} + G^2 N \sum_{s=1}^{t-1} \eta_s \beta^{t-s}. \nonumber
\end{align}
By taking step size of $\eta_t = \frac{\mu + L}{ 2 \mu L} \frac{1}{t}$, we bound the cumulative effect of each term in \eqref{eq:thm:conv:mean2} for $S(T)$, where $T$ is the number of steps we are evaluating the gradient.

\begin{align}
S(T) \leq  \prod_{t=2}^{T} & (1 - \frac{1}{t}) S(0) \label{eq:thm:conv:terms}
\\ & + p N \sum_{t=1}^{T}M_{t-1}^2 \prod_{s=t}^{T}(1 - \frac{1}{s})  \label{eq:dist:conv3} \nonumber \\ 
&+ G^2N \sum_{t=1}^{T} \eta_t^2 \prod_{s= t+1}^{T}(1-\frac{1}{s}) \nonumber \\   
&+ 2G \sqrt{p} N \sum_{t=1}^{T} \eta_t M_t \prod_{s= t+1}^{T}(1-\frac{1}{s}) \nonumber  \\
&+ 2 G \sqrt{p} N \sum_{t=2}^{T} \eta_t \lp \sum_{s=1}^{t-1} M_s \beta^{t-s} \rp \prod_{s= t+1}^{T}(1-\frac{1}{s}) \nonumber  \tag{*1} \\
&+ G^2 N \sum_{t=2}^{T} \eta_t \lp \sum_{s=1}^{t-1} \eta_s \beta^{t-s} \rp \prod_{s= t+1}^{T}(1-\frac{1}{s})   \tag{*2} \nonumber 
\end{align}

The second term \eqref{eq:thm:conv:terms} is dominant in terms of the noise variance, and the noise scales $M_t^2 = \frac{2}{\kappa(\epsilon, \delta)} \lp \frac{\mu + L}{2\mu L} \rp^2 \frac{\sqrt{T}}{t \sqrt{t}}$ are found by minimizing this term while taking condition \eqref{eq:privacy} in Theorem \ref{thm:dp} as the constraint. Therefore,

\begin{align}
S(T) \leq & \frac{S(0)}{T} +  
\frac{4pN}{\kappa(\epsilon, \delta)} \lp \frac{\mu + L}{2\mu L} \rp^2 \label{eq:thm:conv:1} \\ 
&+ G^2 N \lp \frac{\mu + L}{2\mu L} \rp^2 \frac{\log T}{T} \label{eq:thm:conv:2}\\
&+ \frac{8 \sqrt{2p} G N}{\sqrt{\kappa(\epsilon,\delta)}} \lp \frac{\mu + L}{2\mu L} \rp^2 \frac{1}{\sqrt[4]{T} }  \label{eq:thm:conv:3} \\
& + \frac{2 \sqrt{2p} G N}{\sqrt{\kappa(\epsilon,\delta)}}  \lp \frac{\mu + L}{2\mu L} \rp^2   \lp \frac{3}{1-\beta} \rp \frac{1}{\sqrt[4]{T} } \label{eq:thm:conv:4} \\
& + G^2 N \lp \frac{\mu + L}{2\mu L} \rp^2 \lp \frac{1}{1-\beta} \rp \frac{\log T}{T} , \label{eq:thm:conv:5}
\end{align}where we used $\sum_{t=1}^{T} \frac{1}{t} \leq \log(T)$ and $\sum_{t=1}^{T} \frac{1}{t^{\alpha}} \leq \frac{1}{\alpha + 1} T^{\alpha + 1}$ for $\alpha > -1$ to derive \eqref{eq:thm:conv:1}, \eqref{eq:thm:conv:2} and \eqref{eq:thm:conv:3}. Going from $(*1)$ to \eqref{eq:thm:conv:4} follows from
\begin{align}
\sum_{t=2}^{T} \frac{1}{t} \lp \sum_{s=1}^{t-1} \frac{t}{\sqrt{s} \sqrt[4]{s}}  \beta^{t-s }\rp 
&=  \sum_{s=1}^{T-1} \frac{1}{{s^{ 3/4 }}}  \lp \sum_{t=s+1}^{T} \beta^{t-s }\rp \nonumber \\
&\leq  \frac{1}{1-\beta} \sum_{s=1}^{T-1}  \frac{1}{{s^{1/2 + c/4}}} \nonumber \\
&\leq \frac{4}{1-\beta} \sqrt[4]{T}.
\end{align} And we used the following inequality to go from $(*2)$ to \eqref{eq:thm:conv:5},
\begin{align}
\sum_{t=2}^{T} \frac{1}{t} \lp \sum_{s=1}^{t-1} \frac{1}{s}  \beta^{t-s }\rp 
&=  \sum_{s=1}^{T-1} \frac{1}{ s }  \lp \sum_{t=s+1}^{T} \beta^{t-s }\rp \nonumber \\
&\leq  \frac{1}{1-\beta} \sum_{s=1}^{T-1} \frac{1}{s} \leq \frac{1}{1-\beta} \log T.
\end{align}
The result follows immediately by $\Expt[\| \bar{x}(T) - x^\ast \|_2^2] \leq \frac{1}{N} S(T)$.

\section{Numerical Experiments}
\label{sec:experiments}
In this section, we assess the performance of our method on decentralized mean estimation and we demonstrate the effect of privacy parameters, number of gradient evaluation and graph topology on the error. For the simulations, the communication graph is a connected Erdos-Renyi with edge probability of $p_c = 0.6$ and the weight matrix is $W = I - \frac{2}{3\lambda_{\max}\lp L \rp} L$ where $L$ is the Laplacian of the graph.

\subsection{Distributed mean estimation}

Distributed mean estimation is one of the classical problems in the domain of differential privacy. Suppose data points lie in a cube, $\mathcal{X} := [-R, R]^p$, where $p$ is the dimension of the points and $R$ is the length of each side. In this setup, each node has several data points and they aim to collaboratively find the mean while keeping each data private and preserve $(\epsilon, \delta)$-DP against an adversary that oversees all the messages. We can write down this as the following distributed problem:

\begin{align}
    \bar{d} = \min_{x \in \mathcal{X}} f(x) := \frac{1}{2}\sum_{d \in \cup_{i \in [N]} D_i} \| x - d \|_2^2, \label{eq:meanest}
\end{align} where $D_i$ is the set of points stored in node $i$, and $f_i(x) = \frac{1}{2}\sum_{d \in D_i} \| x - d \|_2^2$ for $i \in [N]$. We generate data randomly according to a truncated Gaussian distribution with mean of $0.7 R$ and the unit variance. Sensitive data points are distributed among $10$ nodes each of which has $100$ data points \emph{i.e.}, $|D_i| = 100$. The conditional $l_2$ sensitivity of the distributed algorithm 
\begin{align}
    \Delta(t) &\stackrel{\triangle}= \sup_{i \in [N]} \sup_{D \sim D'} \| x^D_{i}(t) - x^{D'}_{i}(t) \|_2 \nonumber \\
    &\leq \sup_{d, d' \in \mathbb{D}} \eta_t \|d -d' \|_2 \leq 2 R \sqrt{p} \eta_t, \nonumber
\end{align} and we generate Gaussian noise accordingly.

In order to demonstrate the convergence rate of the algorithm, we run the distributed algorithm for different values of $T$, number of gradient descent steps. Figure \ref{fig:meaneast:T} illustrates the effect of $T$ on the normalized error $\frac{\| \bar{x}\lp T \rp - \bar{d}  \|_2^2}{\| \bar{d} \|_2^2}$. It shows that the error reduces until reaching a neighborhood of $x^*$, which agrees with intuition provided by Theorem \ref{thm:conv}.

\begin{figure}
    \centering
    \includegraphics[scale = 0.5]{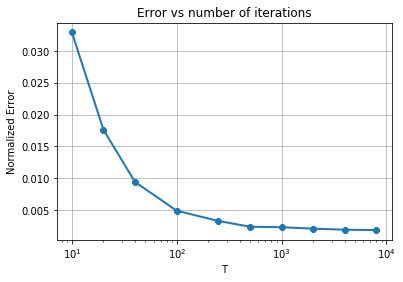}
    \caption[scale=0.5]{The normalized error vs the number of GD steps, $T$ for $\epsilon = 4$ and $\delta = 1/(N \ast N_i) $.}
    \label{fig:meaneast:T}
\end{figure}

Figure \ref{fig:meanest:eps} demonstrates the normalized error vs $\epsilon$ for different values of $\delta \in  \{ 1/(N*N_{i}) , 1/(N*N_{i})^2, 1/(N*N_{i})^3 \}$ where $N_i$ denotes the number of data points in each node, and $T = 1000$. We observe that for a fixed value of $\delta$ by strengthening the privacy guarantee the error increases $O(\frac{1}{\epsilon^2})$.

\begin{figure}
    \centering
    \includegraphics[scale = 0.5]{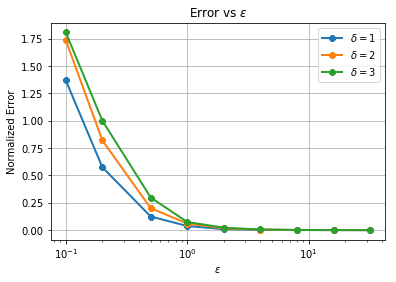}
    \caption{The normalized error of the distributed mean estimation vs $\epsilon$ for a fixed number of nodes and data points per node.}
    \label{fig:meanest:eps}
\end{figure}

To observe the effect of the graph topology, we fix the privacy parameters and vary the connectivity probability of the underlying graph by choosing $p_c \in \{0.1, 0.3, 0.6, 1 \}$. Figure \ref{fig:meaneast:prob} illustrates the effect of connectivity on error of an individual node $\frac{ \| x_i\lp T \rp - \bar{d}  \|_2^2}{\| \bar{d} \|_2^2}$.

\begin{figure}
    \centering
    \includegraphics[scale = 0.5]{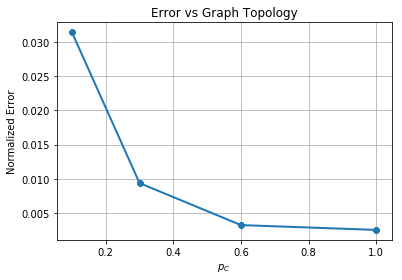}
    \caption[scale=0.5]{As the connectivity of the graph increases, we observe a decrease in the error of the first node.}
    \label{fig:meaneast:prob}
\end{figure}

Figure \ref{fig:meaneast:N_agent} illustrates the regret bound when the number of data points per each node is increasing. We observe a gain in the utility bound that is inline with Theorem \ref{thm:conv}. Increasing the number of data points doesn't increase the conditional sensitivity while making the Gradient bigger hence the effective added noise is reduced.

\begin{figure}
    \centering
   \includegraphics[scale = 0.5]{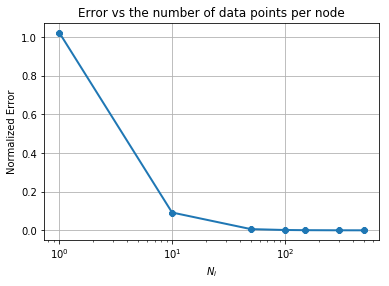}
    \caption[scale=0.5]{The error versus number of data points per each node for a fixed $T$, $\epsilon = 4$ and $\delta = 1/(N \ast N_i) $.}
    \label{fig:meaneast:N_agent}
\end{figure}

\section{conclusion}
\label{sec:conclusion}

In this work, we studied the consensus-based distributed optimization algorithm when the data points are distributed across several trusted nodes and the privacy of each data point is important against an adversary that oversees communications across nodes. In order to protect the privacy of users, each node perturbs it's local state with an additive noise before sending it out to its neighbors. Differential privacy is a rigorous privacy criterion for data analysis that provides meaningful guarantees regardless of what an adversary knows ahead of time about individuals' data. We considered $(\epsilon, \delta)$-DP as the privacy measure and we derived the amount of noise needed in order to guarantee privacy of each data point. We further showed that the parameters converge to a neighborhood of the optimal point, and the size of the neighborhood is proportional to the privacy metrics.

\bibliographystyle{ieeetr}
\bibliography{references}

\appendix

\begin{proof}[Proof of Proposition \ref{prop:conditionalsens}]
Recall from \eqref{eq:def:GD} that $z_i(t)$ is a function of $\{ \mathbf{y}(t) \}_{t = 1}^{t}$, therefore conditioned on the same set of messages $\{ \mathbf{y}(t) \}_{t = 1}^{t}$ we have the following , 
\begin{align}
    \| &x_i^D(t) - x_i^{D'}(t) \|_2 \label{eq:prop:conditionalsens:proof} \\ 
    &= \| \Proj_{\mathcal{X}} \lp z_i^D(t) - \eta_t \nabla f_i^D (z_i^D(t)) \rp \nonumber \\ 
    &\hspace{0.18\textwidth}- \Proj_{\mathcal{X}} \lp z_i^{D'}(t) - \eta_t \nabla f_i^{D'} (z_i^{D'}(t)) \rp \| \nonumber \\
    & \stackrel{(a)}\leq \eta_t \| \nabla f_i^D (z_i^D(t)) -  \nabla f_i^{D'} (z_i^{D'}(t)) \| \stackrel{(b)}\leq 2 \eta_t G, \nonumber
\end{align} where $(a)$ follows from non-expansiveness of the projection operator and Assumption \ref{asump:bddgrad} implies $(b)$. Taking the supermom from both sides of \eqref{eq:prop:conditionalsens:proof} implies the result directly.
\end{proof}

\begin{proof}[Proof of Theorem \ref{thm:dp}]
In order to prove Algorithm \ref{alg:PDGD} satisfies $(\epsilon, \delta)$-DP, we check condition \eqref{eq:thm:dp} in Proposition \ref{prop:dp:equiv}. Recall that the adversary can only observe messages among nodes and the privacy loss random variable is a function of these observations. We derive an analytical expression for the privacy loss random variable and bound it using concentration inequalities.

Note that we can write the pdf of $\mathbf{y}(1), \cdots, \mathbf{y}(T)$ as follows:
\begin{align*}
\mbox{pdf}_D(\mathbf{y}(1), \cdots, \mathbf{y}(T) ) &= \prod_{t} \mbox{pdf}_{D}(\mathbf{y}(t+1) | \mathbf{y}(1), \cdots, \mathbf{y}(t) )  \\
& \stackrel{(a)}= \prod_{t} \mbox{pdf}_{D}(\mathbf{y}(t+1) | \mathbf{y}(t) )  \\
& \stackrel{(b)}= \prod_{t} \prod_{k \in [N]} \mbox{pdf}_{D}(y_{k}(t+1) | \mathbf{y}(t) ) \\
& \stackrel{(c)}= \prod_{t} \prod_{k \in [N]} p_{M_t}( y_{k}(t+1) - x_{k}(t) )  ), 
\end{align*} where $(a)$ follows since the randomness comes from the additive noise \eqref{eq:def:noise} and $\mathbf{y}(t+1)$ conditioned on $\mathbf{n}(t)$ (and therefore $\mathbf{y}(t)$) is independent of all the previous coin tosses of the algorithm. Noise injected independently across nodes which implies $(b)$. In $(c)$, we wrote the conditional pdf of $y_k(t+1)$ in terms of density function of a Gaussian, $p_{M_t}$.

Let us denote the privacy loss random variable for $T$ stages with $c( \mathbf{y}(1), \cdots, \mathbf{y}(T) )$. We distinguish the variables associated with neighboring database $D'$ with $'$. In the context of this problem, neighboring databases differ in at most one data point, \emph{i.e.}, at most one node may have a different function. We denote this node with $k^* \in [N]$.
\begin{align}
&c( \mathbf{y}(1), \cdots, \mathbf{y}(T) )  = \log \frac{\mbox{pdf}_D(\mathbf{y}(1), \cdots, \mathbf{y}(T) )}{\mbox{pdf}_{D'}(\mathbf{y}(1), \cdots, \mathbf{y}(T) )} \label{eq:thm:dp:c_def} \\
&= \sum_{t} \sum_{k \in [N]} \frac{ - \| y_{k}(t+1) - x_{k}(t)  \|^2 }{ 2M_t^2 } + \frac{  \| y_{k}(t+1) - x'_{k}(t)  \|^2 }{ 2M_t^2 } \nonumber \\
&= \sum_{t} \sum_{k \in [N]} \frac{  \| x_{k}(t) - x'_{k}(t)  \|^2 }{ 2M_t^2 } \nonumber \\
&\hspace{0.18\textwidth}+ \frac{  2\langle y_{k}(t+1) - x_{k}(t),  x_{k}(t) - x'_{k}(t)  \rangle }{ 2M_t^2 } \nonumber \\
&= \sum_{t} \sum_{k \in [N]} \frac{  \| x_{k}(t) - x'_{k}(t)  \|^2 }{ 2M_t^2 } +\sum_{t} \sum_{k \in [N]} \langle n_{k}(t), \frac{ ( x_{k}(t) - x'_{k}(t) ) }{ M_t^2 } \rangle \nonumber \\
&\stackrel{(a)}= \sum_{t} \frac{  \| x_{k^*}( t ) - x'_{k^*}( t )  \|^2 }{ 2M_t^2 } + \sum_{t} \langle n_{k^*}(t), \frac{ ( x_{k^*}(t) - x'_{k^*}(t) ) }{ M_t^2 } \rangle, \nonumber
\end{align} where $(a)$ follows because only one of the cost functions is different among neighboring databases, therefore at most one of the these terms is non-zero (note that we are conditioning on the same observations $\{ \mathbf{y}(t) \}$ across two neighboring problems). 

Recall from the definition of $\Delta^2(t)$ \eqref{eq:def:conditionalsens},
\begin{align}
\sum_{t} \frac{  \| x_{k}( t  ) - x'_{k}(t)  \|^2 }{ M^2(t) } \leq \sum_{t} \frac{  \Delta^2(t) }{ M^2(t) } \stackrel{\triangle}= \alpha, \label{eq:thm:dp:alphadef}
\end{align} therefore by putting \eqref{eq:thm:dp:c_def} and \eqref{eq:thm:dp:alphadef} together, 
\begin{align}
    c( \mathbf{y}(1), \cdots, \mathbf{y}(T) ) \leq \frac{\alpha}{2} + \sum_{t=1}^{T} \langle n_{k^*}(t), \frac{ ( x_{k^*}(t) - x'_{k^*}(t) ) }{ M_t^2 } \rangle. \label{eq:thm:dp:proof:c}
\end{align} In order to bound $c$ we first show the second term in \eqref{eq:thm:dp:proof:c} $C_T \stackrel{\triangle}= \sum_{t=1}^{T} \langle n_{k^*}(t), \frac{ ( x_{k^*}(t) - x'_{k^*}(t) ) }{ M_t^2 } \rangle$ is sub-Gaussian.
\begin{definition}[sub-Gaussian]
A zero mean random variable $X$ is sub-Gaussian\footnote{There exists several equivalent definitions of sub-Gaussianity in the literature, see for example \cite{HDP} Proposition 2.5.2} if for some $\sigma > 0$,
\begin{align}
    \Expt e^{\lambda X} \leq e^{\sigma^2 \lambda^2/2}, \forall \lambda \in \mathbb{R}. \label{eq:def:subgaussian}
\end{align}
\end{definition}

\begin{lemma}
The second term in \eqref{eq:thm:dp:proof:c}, $C_T$ is sub-Gaussian with the parameter bounded by $\sqrt{\alpha}$, where $\alpha$ is defined in \eqref{eq:thm:dp:alphadef}. \label{lemma:azuma}
\end{lemma}

\begin{proof}
Recall the definition of $C_T$ and let us define $w_t$ as follows:
\begin{align*}
C_T &\stackrel{\triangle}= \sum_{t=1}^{T} \langle n_{k^*}(t), \frac{ ( x_{k^*}(t) - x'_{k^*}(t) ) }{ M_t^2 } \rangle, \\
w_t &\stackrel{\triangle}= \Expt[C_T | \mathcal{F}_{t} ] -  \Expt[C_T | \mathcal{F}_{t-1} ] = \langle n_{k^*}(t), \frac{ ( x_{k^*}(t) - x'_{k^*}(t) ) }{ M_t^2 } \rangle,
\end{align*} where $\mathcal{F}_{t}$ is the sigma algebra generated by $n_{k}(1), \cdots, n_{k}(t)$. Note that $w_t$ is conditionally sub-Gaussian with $\sigma_t \leq \Delta(t)/M_t$ since:
\begin{align}
\Expt[ e^{\lambda w_t} | n_{k}(1), \cdots, n_{k}(t-1) ] &\stackrel{(a)}= e^{ \frac{ \| x_{k}(t) - x'_{k}(t) \|^2}{M^2_t} \lambda^2/2 } \nonumber \\
& \stackrel{(b)}\leq e^{ \frac{ \Delta^2(t)}{ M^2_t} \lambda^2/2 }, \quad \forall \lambda \in \mathbb{R} \label{eq:lemma:azuma}
\end{align} here $(a)$ follows since $n_{k}(t)$ is an i.i.d. zero mean random Gaussian vector with variance of $M_t^2$ and it's independent of $\mathcal{F}_{t-1}$, and definition of $\Delta^2(t)$ implies $(b)$. Taking the conditional expectation of $C_T$ with respect to $\mathcal{F}_{t-1}$ implies,
\begin{align}
    \Expt[ e^{\lambda C_T} | \mathcal{F}_{T-1} ] &=  e^{\lambda \sum_{t = 1}^{T-1} w_t } \Expt[e^{\lambda w_T}| \mathcal{F}_{T-1}] \nonumber \\
    &\stackrel{(a)}\leq e^{\lambda \sum_{t = 1}^{T-1} w_t } e^{ \frac{ \Delta^2(T)}{ M^2_T} \lambda^2/2 }, \forall \lambda \in \mathbb{R} \label{eq:lemma:azuma:expt1}
\end{align} where $(a)$ follows from \eqref{eq:lemma:azuma}. By repeating the same argument and taking the conditional expectation with respect to $\mathcal{F}_{T-1}$, $\mathcal{F}_{T-2}$ up until $\mathcal{F}_{1}$ we conclude that,
\begin{align*}
    \Expt[ e^{\lambda C_T} ]  &\leq e^{ \sum_{t = 1}^{T} \frac{ \Delta^2(t)}{ M^2_t} \lambda^2/2 } \\
    &\leq e^{\alpha \lambda^2/2}, \quad \forall \lambda \in \mathbb{R}.
\end{align*}
\end{proof}
In order to bound the privacy random variable, we use the following tail bound for sub-Gaussian random variables.
\begin{proposition}
Assume $X$ is a random variable satisfying \eqref{eq:def:subgaussian}. Then we have
\begin{align}
    \Prob{X > t} \leq e^{- \frac{t^2}{2 \sigma^2}}, \forall t \geq 0, \label{eq:prop:subgassian}
\end{align} \label{prop:subgaussian}
\end{proposition}
\begin{proof}
Proof follows directly by applying Chernoff bound along with condition in \eqref{eq:def:subgaussian}.
\end{proof}

Now we are ready to prove the claim. We show that $\Prob{ | c | \geq \epsilon } \leq \delta$ where $c \stackrel{\triangle}= c( \mathbf{y}(1), \cdots, \mathbf{y}(T) )$ as in \eqref{eq:thm:dp:c_def}. Note that,

\begin{align}
\Prob{ c \geq + \epsilon } &\stackrel{(a)}\leq \Prob{ \frac{\alpha}{2} + C_T \geq \epsilon } \stackrel{(b)}\leq e^{- ( \epsilon - \frac{\alpha}{2} )^2/{2 \alpha}}, \label{eq:thm:dp:proof:conceq1} \\
\Prob{ c \leq -\epsilon } &\stackrel{(c)}\leq \Prob{ C_T \leq - \epsilon  } \stackrel{(d)}\leq e^{- \frac{ \epsilon ^2 }{2 \alpha}}, \label{eq:thm:dp:proof:conceq2}
\end{align} where $(a)$ follows from inequality \eqref{eq:thm:dp:proof:c}, $(b)$ is the direct consequence of Proposition \ref{prop:subgaussian}, $(c)$ comes from $c \geq C_T$ and we use the fact that $-C_T$ is a sub-Gaussian random variable along with Proposition \ref{prop:conditionalsens} in $(d)$. Therefore, 
\begin{align}
\Prob{ | c | \geq \epsilon } &= \Prob{ c \geq \epsilon } + \Prob{ c \leq -\epsilon } \leq 2 e^{- \frac{( \epsilon - \frac{\alpha}{2} )^2 }{2 \alpha}} \stackrel{(a)}\leq \delta. \label{eq:thm:dp:proof:last}
\end{align}  It is straightforward to verify that $(a)$ holds for
\begin{align*}
0 \leq \alpha \leq 2\epsilon + 4 \log\frac{2}{\delta} - 2\sqrt{\left( \epsilon + 2 \log\frac{2}{\delta} \right)^2 - \epsilon^2} \tag{*}
\end{align*}
By using the inequality $a/2 \leq 1 - \sqrt{1 - a}$ for $|a|\leq 1$ we conclude that an specific choice of $\alpha = \frac{\epsilon^2}{\epsilon + 2 \log \frac{2}{\delta}}$ in the statement of Theorem \ref{thm:dp} lies in $(*)$.
\end{proof}

\begin{proof}[Proof of Lemma \ref{lemma:dist_mean}]
Compared to the classical DGD (see for example \cite{wotao16, dist_09}) we have an additional projection operator that we need to take into account. In this case, with a minor modification we can still bound the distance of individual's $z_i$ to the average parameter.
Let us rewrite $x_i(t)$ as follows:
\begin{align}
x_i(t) = \Proj_{\mathcal{X}} \lp z_i(t) - \eta_t \nabla f_i(z_i(t)) \rp &= z_i(t)  + v_i(t) \label{eq:dis_mean1} \\
&= \hat{z}_t(t) + u_i(t) + v_i(t), \nonumber
\end{align} where $\hat{z}_i(t) \stackrel{\triangle}= \sum_{j \in \mathcal{N}_i} w_{i j} y_j(t)$ and $u_i(t)$ and $v_i(t)$ are defined as 
\begin{align}
v_i(t) &\stackrel{\triangle}= \Proj_{\mathcal{X}} \lp z_i(t) - \eta_t \nabla f_i(z_i(t)) \rp - z_i(t), \label{eq:lemma:dist_mean:defv} \\
u_i(t) &\stackrel{\triangle}= z_i(t) -\hat{z}_i(t). \label{eq:lemma:dist_mean:defu}
\end{align} We use the notation $\mathbf{\hat{z}}(t) = [\hat{z}_1(t); \cdots, \hat{z}_N(t)]$, $\mathbf{u}(t) = [u_1(t); \cdots; u_N(t)]$ and $\mathbf{v}(t) = [v_1(t); \cdots; v_N(t)]$. In the rest of proof, we rewrite the mean parameter using Kronecker product $\bar{z}(t) = \frac{1}{N} \lp \mathbf{1}_N^T \otimes I_P \rp \mathbf{z}(t)$. In order to bound $\| \mathbf{z}(t) - \frac{1}{N} \lp \mathbf{1}_N \mathbf{1}_N^T  \otimes I_{p} \rp \mathbf{z}(t) \|$, we first  present a lemma that bounds $\| \mathbf{\hat{z}}(t) - \frac{1}{N} \lp \mathbf{1}_N \mathbf{1}_N^T \otimes I_{p} \rp \mathbf{\hat{z}}(t) \|$.

\begin{lemma}
Under Assumption \ref{asump:bddgrad}, at any time $t \leq T$, the following holds:
\begin{align}
 \| \mathbf{\hat{z}}(t) - \frac{1}{N} \lp \mathbf{1}_N \mathbf{1}_N^T  \otimes I_{p} \rp &\mathbf{\hat{z}}(t) \| \\ 
 &\leq 2\sum_{s=1}^{t-1} \beta^{t-s} \| \mathbf{n}(s) \|  + \sqrt{N} G \sum_{s=1}^{t-1} \eta_s  \beta^{t-s}. \nonumber
\end{align} \label{lemma:zhat}
\end{lemma}

\begin{proof}
Observer that plugging \eqref{eq:dis_mean1} into \eqref{eq:def:consensus} we can write:
\begin{align}
\mathbf{\hat{z}}(t) = \lp W \otimes I_{p} \rp ( \mathbf{\hat{z}}(t-1) + & \mathbf{u}(t-1) + \mathbf{v}(t-1) )  \nonumber \\
+ & \lp W \otimes I_{p} \rp \mathbf{n}(t-1), \label{eq:lemma:zhat:proof:recur}
\end{align} where $\mathbf{u}$ and $\mathbf{v}$ are defined according to \eqref{eq:lemma:dist_mean:defu} and \eqref{eq:lemma:dist_mean:defv} respectively.

We proceed by bounding each of $\mathbf{u}(t)$ and $\mathbf{v}(t)$ for $t \leq T$. It is clear that both terms are zero in Stage II of the algorithm.
\noindent\textbf{Bounding $\mathbf{u}$:} Note that
\begin{align}
\| v_i (t) \|_2 &= \| \Proj_{\mathcal{X}} \lp z_i(t) - \eta_t \nabla f_i(z_i(t)) \rp - z_i(t)\| \nonumber \\ 
&\stackrel{(a)}= \| \Proj_{\mathcal{X}} \lp z_i(t) - \eta_t \nabla f_i(z_i(t)) \rp - \Proj_{\mathcal{X}}z_i(t) \| \nonumber \\
&\stackrel{(b)}\leq \| \eta_t \nabla f_i(z_i(t)) \|_2 \nonumber \\
&\stackrel{(c)}\leq  G \eta_t, \nonumber
\end{align} where $(a)$ holds since $z_i(t) \in \mathcal{X}$, $(b)$ follows from non-expansiveness property of the Euclidean projection and Assumption \ref{asump:bddgrad} implies $(c)$. Therefore,
\begin{align}
\| \mathbf{v} (t) \|_2 &=\sqrt{ \sum_{i}^{N} \| v_i(t) \|_2^2  } \leq  G \sqrt{N} \eta_t. \label{eq:lemma:zhat:v}
\end{align} In order to bound $\mathbf{u}(t)$,
\begin{align}
\| u_i(t) \|_2 = \| z_i - \hat{z}_i \|_2 &= \| \Proj_{\mathcal{X}} \hat{z}_i - \hat{z}_i  \|_2 \nonumber \\ &\stackrel{(a)}\leq  \| \lp W_i \otimes I_{p} \rp \mathbf{n}(t) \|_2, \nonumber
\end{align} where $(a)$ follows from $\| \Proj_{\mathcal{X}} a - b \| \leq \| c - b \|, \forall c \in \mathcal{X}$\footnote{This property follows directly from the definition of projection \eqref{eq:def:proj}} and the fact that $\sum_{j}w_{i,j} x_j(t) \in \mathcal{X}$, $W_i$ is the $i$th row of the weight matrix. Therefore, we have the following bound on $\| \mathbf{u} \|^2$:
\begin{align}
\| \mathbf{u} \|^2 = \| \lp W \otimes I_{p}  \rp \mathbf{n}(t) \|^2 \leq  \| \mathbf{n}(t) \|^2. \label{eq:lemma:zhat:u}  
\end{align} The rest follows from the classical case \cite{dist_09, wotao16}, by bounding the deviation of $\hat{z}_i$ from $\frac{1}{N} \sum_{i = 1}^{N} \hat{z}_i$ assuming zero initial states\footnote{Without loss of generality we assume the initial condition is zero, otherwise there is an extra term that goes to zero exponentially fast.} by expanding \eqref{eq:lemma:zhat:proof:recur},
\begin{align}
 \| \mathbf{\hat{z}}(t) - &\frac{1}{N} \lp \mathbf{1}_N \mathbf{1}_N^T \otimes I_{p} \rp \mathbf{\hat{z}}(t) \| \nonumber \\ 
 &\stackrel{(a)}= \| \sum_{s=0}^{t-1} \lp W^{t-s} \otimes I_{p}  - \frac{1}{N} \mathbf{1}_N \mathbf{1}_N^T \otimes I_{p} \rp \mathbf{n}(s) \nonumber \nonumber \\ 
 & \hspace{0.05\textwidth} + \sum_{s=1}^{t-1}   \lp W^{t-s} \otimes I_{p}  - \frac{1}{N} \mathbf{1}_N \mathbf{1}_N^T \otimes I_{p} \rp \lp \mathbf{u}(s) + \mathbf{v}(s)  \rp \| \nonumber \\
&\stackrel{(b)}\leq \sum_{s=0}^{t-1}  \| W^{t-s} \otimes I_{p}  - \frac{1}{N} \mathbf{1}_N \mathbf{1}_N^T \otimes I_{p} \| \| \mathbf{n}(s) \|  \nonumber \\
 & \hspace{0.07\textwidth} +\sum_{s=1}^{t-1} \| W^{t-s} \otimes I_{p}  - \frac{1}{N} \mathbf{1}_N \mathbf{1}_N^T \otimes I_{p} \| \|  \mathbf{u}(s) + \mathbf{v}(s) \| \nonumber \\
 &\stackrel{(c)}\leq \sum_{s=0}^{t-1} \beta^{t-s} \| \mathbf{n}(s) \|  +  \sum_{s=1}^{t-1}  \beta^{t-s} \| \mathbf{u}(s) \| +  \sum_{s=1}^{t-1}  \beta^{t-s} \| \mathbf{v}(s) \| \nonumber \\
 &\stackrel{(d)}\leq \sum_{s=0}^{t-1} \beta^{t-s} \| \mathbf{n}(s) \| +  \sum_{s=1}^{t-1} M_s \beta^{t-s} \| \mathbf{n}(s) \| +  \sqrt{N} G  \sum_{s=1}^{t-1} \eta_s  \beta^{t-s}  \nonumber \\
 &= \sum_{s=1}^{t-1} 2\beta^{t-s} \| \mathbf{n}(s) \|  + \sqrt{N} G \sum_{s=1}^{t-1} \eta_s  \beta^{t-s},
\end{align} where $(a)$ holds since $W$ is doubly stochastic, $(b)$ follows from the triangle inequality together with the inequality $\|Ax\| \leq \|A\| \|x\|$ where $\|A\|$ denotes the operator norm\footnote{$\|Ax\| \leq \|A\| \|x\|$ holds for matrix $A \in \mathbb{R}^{n \times n}$ and $x \in \mathbb{R}^n$ where $\|A\|$ is the operator norm of a matrix.}, we use the spectral property of the weight matrix in $(c)$ and $\beta$ is defined as the second largest eigenvalue of $W$, we plugged in \eqref{eq:lemma:zhat:u} and \eqref{eq:lemma:zhat:v} to derive $(d)$.
\end{proof}

Having set Lemma \ref{lemma:zhat}, we bound the the deviation of $z_i(t)$ from the mean parameter for Stage I as follows.
\begin{align}
\| \mathbf{z}(t)& - \frac{1}{N} \lp \mathbf{1}_N \mathbf{1}_N^T  \otimes I_{p} \rp \mathbf{z}(t) \|  \\
&= \| \lp I_{pN} - \frac{1}{N} \lp \mathbf{1}_N \mathbf{1}_N^T  \otimes I_{p} \rp \rp \lp \mathbf{\hat{z}}(t) + \mathbf{u}(t) \rp \| \nonumber \\
&\leq  \| \mathbf{\hat{z}}(t) - \frac{1}{N} \lp \mathbf{1}_N \mathbf{1}_N^T  \otimes I_{p} \rp \mathbf{\hat{z}}(t) \| \nonumber\\
&\hspace{0.12\textwidth}+ \| \lp I_{pN} - \frac{1}{N} \lp \mathbf{1}_N \mathbf{1}_N^T  \otimes I_{p} \rp \rp  \mathbf{u}(t) \| \nonumber \\
&\leq \sum_{s=1}^{t-1} 2 M_s  \beta^{t-s} \| \mathbf{n}(s) \|  + \sqrt{N} G \sum_{s=1}^{t-1} \eta_s  \beta^{t-s} + \| \mathbf{u}(t) \|, \nonumber
\end{align} where we used Lemma \ref{lemma:zhat}, \eqref{eq:lemma:zhat:u} along with triangle inequality and the inequality $\|Ax\| \leq \|A\| \|x\|$.

In Stage II, $\mathbf{x}(t) = \lp W \otimes I_P \rp \mathbf{x}(t-1)$ for $t > T$, therefore
\begin{align}
    \| \mathbf{x}(t) - \frac{1}{N} &\lp \mathbf{1}_N \mathbf{1}_N^T  \otimes I_{p} \rp \mathbf{x}(t) \| \nonumber \\
    &= \| ( W^{t-T}\otimes I_p - \frac{1}{N} \lp \mathbf{1}_N \mathbf{1}_N^T  \otimes I_{p} \rp ) \mathbf{x}(T) \|  \nonumber
    \\ &\leq \beta^{t-T} \mathbf{x}(T),
\end{align} where we use the spectral property of $W$.

\end{proof}

\begin{proof}[Proof of Lemma \ref{lemma:sum:noise}]
Let us define $\hat{z}_i(t) \stackrel{\triangle}= \sum_{j \in \mathcal{N}_i} w_{i j} y_j(t)$, \emph{i.e.}, $z_i(t) = \Proj_{\mathcal{X}} \hat{z}_i(t) $ therefore
\begin{align}
    \| z_i(t) - x \| &\stackrel{(a)}= \| \Proj_{\mathcal{X}} \hat{z}_i(t) - \Proj_{\mathcal{X}} x \| \nonumber \\
    &\stackrel{(b)}\leq \| \hat{z}_i(t) - x \|^2, \label{eq:lemma:sum:noise:proof1}
\end{align} where $(a)$ follows since $x \in \mathcal{X}$ and we used non-expansiveness property of the projection in $(b)$. Summing up both sides of \eqref{eq:lemma:sum:noise:proof1} results in
\begin{align}
\sum_{i \in [N] } \| z_i(t) - x \|^2 &\leq \sum_{i \in [N] } \| \hat{z}_i(t) - x \|^2  \label{eq:lemma:sum:noise:proof2}  \\
&= \| \mathbf{\hat{z}}(t) - \mathbf{1}_N \otimes x  \|^2  \nonumber \\
&\stackrel{(a)}= \| \lp W \otimes I_{p} \rp y(t) - \mathbf{1}_N \otimes x  \|^2 \nonumber \\
&\stackrel{(b)}= \| \lp W \otimes I_{p} \rp \lp y(t) -  \mathbf{1}_N \otimes x \rp \|^2 \nonumber \\
&\stackrel{(c)}\leq  \| y(t) - \mathbf{1}_N \otimes x  \|^2 = \sum_{i \in [N]} \| y_i(t) - x \|^2, \nonumber
\end{align} where $(a)$ follows directly from definition of $\hat{z}_i(t)$, $W$ being doubly stochastic implies $(b)$ and $(c)$ is from the spectral properties of $W$. 

Note that,
\begin{align}
    \| y_i(t) - x \|^2 = \| x_i(t-1) - x \|^2 &+ \| n_i(t) \|^2  \label{eq:lemma:sum:noise:proof3} \\
    &+ 2 \langle x_i(t-1) - x, n_i(t) \rangle. \nonumber
\end{align} The proof is complete by plugging \eqref{eq:lemma:sum:noise:proof3} into \eqref{eq:lemma:sum:noise:proof2} together and taking the expectation from both sides.

\end{proof}

\end{document}